\documentclass[a4paper,11pt]{article}
\usepackage{booktabs} 
\usepackage{graphicx}
\usepackage{amsmath}
\usepackage{amssymb}
\usepackage[ruled,vlined]{algorithm2e}
\usepackage{amsthm}
\usepackage{amsfonts}
\usepackage{xspace}
\usepackage{pgfplots}
\usetikzlibrary{patterns}
\usepackage{enumitem}
\usepackage{mathtools}
\usepackage[noend]{algpseudocode}
\usepackage[toc,page]{appendix}
\pgfmathdeclarefunction{gauss}{2}{%
	\pgfmathparse{1/(#2*sqrt(2*pi))*exp(-((x-#1)^2)/(2*#2^2))}%
}
\providecommand{\keywords}[1]{\textbf{\textit{Index terms---}} #1}

\newtheorem{theorem}{Theorem}
\newtheorem{lemma}{Lemma}

\newcommand{\ea}{\text{\sc EA}\xspace} 

\setlist[itemize]{leftmargin=*}
\setlist[enumerate]{leftmargin=*}

\usepackage[switch]{lineno} 

\DeclarePairedDelimiter\ceil{\lceil}{\rceil}

\DeclareMathOperator{\unif}{Unif} 

\newcommand{\mas}{\text{\sc MAs}\xspace} 
\newcommand{\ls}{\text{\sc LocalSearch}\xspace} 
\newcommand{\om}{\text{\sc OneMax}\xspace} 
\newcommand{\hurdle}{\text{\sc Hurdle}\xspace} 

\newcommand{\mutate}{\text{\sc Mutate}\xspace}
\newcommand{\flip}{\text{\sc Flip}\xspace}
\newcommand{\curbestfit}{\text{\sc CurBestFit}\xspace} 
\newcommand{\curbestind}{\text{\sc CurBestInds}\xspace}
\newcommand{\permute}{\text{\sc Per}\xspace} 
\newcommand{\betterfound}{\text{\sc BetterFound}\xspace}
\newcommand{\rls}{\text{\sc RLS}\xspace} 
\newcommand{\bils}{\text{\sc BILS}\xspace} 
\newcommand{\fils}{\text{\sc FILS}\xspace} 
\newcommand{\rem}{\text{rem}\xspace} 
\newcommand{\ie}{\text{i.e.}\xspace} 
\newcommand{\Prob}[1]{\text{Prob}\left(#1\right)}
\renewcommand{\textbf}[1]{#1}

\usepackage{subcaption}
\usepackage{wrapfig}

\usepackage[numbers]{natbib}
\title{Memetic Algorithms Beat Evolutionary Algorithms
	on the Class of Hurdle Problems    
    \thanks{Preliminary version of this work will appear 
in the Proceedings of the 2018 Genetic and 
Evolutionary Computation Conference (GECCO 2018)} }
    
\author{
  Phan Trung Hai Nguyen\\
  School of Computer Science\\
  University of Birmingham\\ 
  \texttt{p.nguyen@cs.bham.ac.uk}
  \and
  Dirk Sudholt\\
  Department of Computer Science\\
  University of Sheffield\\
  \texttt{d.sudholt@sheffield.ac.uk}
}

\begin{document}

\maketitle

\begin{abstract}
   Memetic algorithms are popular hybrid search heuristics that integrate local search into the search process of an evolutionary algorithm in order to combine the advantages of rapid exploitation and global optimisation. However, these algorithms are not well understood and the field is lacking a solid theoretical foundation that explains when and why memetic algorithms are effective.

We provide a rigorous runtime analysis of a simple memetic algorithm, the (1+1)~MA, on the \hurdle problem class, a landscape class of tuneable difficulty that shows a ``big valley structure'', a characteristic feature of many hard problems from combinatorial optimisation. The only parameter of this class is
the hurdle width $w$, which describes the
length of fitness valleys that have to be overcome. We show that the (1+1)~EA requires $\Theta(n^w)$ expected function evaluations to find the optimum, whereas the (1+1)~MA with best-improvement and first-improvement local search can find the optimum in $\Theta(n^2 + n^3/w^2)$ and $\Theta(n^3/w^2)$ function evaluations, respectively. Surprisingly, while increasing the hurdle width makes the problem harder for evolutionary algorithms, the problem becomes easier for memetic algorithms.
We discuss how these findings can explain and illustrate the success of memetic algorithms for problems with big valley structures.
\end{abstract}

 

\section{Introduction}


\subsection{Motivation}

Memetic Algorithms (\mas), also known as evolutionary/genetic local search or global-local search hybrids, are hybrid stochastic search methods that
incorporate one or more intensifying local search algorithms into
an evolutionary framework. The motivation behind this hybridisation
is to create a new algorithm that combines the exploration capabilities of an evolutionary algorithm with the efficiency and exploitation capabilities of local search. There are many examples where this strategy has proven effective; see e.\,g.~\cite{Neri2012,Neri2012a}.

In~\cite{Sudholt2011c}, three advantages for memetic algorithms are pointed out:
\begin{enumerate}
	\item Local search can quickly find solutions of high quality due to its rapid exploitation.
	\item Selection is only performed after local search has had a chance to improve on new offspring; this is beneficial for low-fitness offspring located in the basin of attraction of a high-fitness local optimum, as in a conventional evolutionary algorithm such low-fitness offspring would be removed by selection. This effect is particularly visible for constrained problems, where penalties are used for violated constraints, and local search can act as a repair mechanism~\cite{Sudholt2011c}.
	\item Local search can include problem-specific knowledge; this is often possible since local search strategies are typically easy to design, even when it is hard to design a global problem-specific strategy~\cite{Sudholt2011c}.
\end{enumerate}

A challenge when dealing with memetic algorithms and
hybrid algorithms, in general, is that the search dynamics can be very hard to understand, in particular due to the interplay of different operators. It is not well understood when and why memetic algorithms perform well, when they do not, and how to design memetic algorithms most effectively for a problem in hand. Most work in this area is empirical and the theory of memetic algorithms is still in its infancy.

There are many different variants of memetic algorithms, from algorithms that only rarely apply local search, with a fixed \emph{local search frequency} to \emph{iterated local search} algorithms where local search is applied in every generation~\cite{Lourencco2002}.
In the latter scenario, local search turns all search points into local optima, and evolution acts on the sub-space of local optima. The hope is that mutation can lead a memetic algorithm to leave its current local optimum, and to reach the basin of attraction of a better one.

We demonstrate that this strategy works very effectively on a class of problems introduced by Pr{\"u}gel-Bennett~\cite{PRUGELBENNETT2004135} as example problems where genetic algorithms using crossover perform better than hill climbers. The \hurdle problem class (formally defined in Section~\ref{hurdle-problems}) is a function of unitation\footnote{A function of unitation is a function that only depends on the number of ones.} with an underlying gradient leading towards the global optimum and a number of ``hurdles'' that have to be overcome. These hurdles consist of a local optimum and a fitness valley that has to be overcome to reach the next local optimum. The distance between local optima is a parameter~$w$ called the \emph{hurdle width}, and it can be used to parameterise the width of the fitness valleys. For simple evolutionary algorithms like the (1+1)~EA, a larger hurdle width makes the problem harder. This effect was analysed in~\cite{PRUGELBENNETT2004135} with non-rigorous arguments based on simplifying assumptions, that led to approximations of the expected time for finding the global optimum.

Here we provide a rigorous analysis for the expected optimisation time of the (1+1)~EA: we give a tight bound of $\Theta(n^w)$ for the expected optimisation time\footnote{See~\cite[Chapter 3]{Cormen2009} for a definition of asymptotic notation and symbols~$\Theta, O, o, \Omega, \omega$.}, confirming that the performance degrades very rapidly with increasing hurdle width. For hurdle widths growing with~$n$, $w = \omega(1)$, this expected time is superpolynomial and hence intractable.

In contrast, we show that memetic algorithms perform very effectively on this problem class due to their combination of evolutionary operators and local search. We study a simple iterated local search algorithm called (1+1)~MA with two different local searches, \emph{First-Improvement Local Search} (\fils) and \emph{Best-Improvement Local Search} (\bils) \cite{Dinneen2013}, and show that the (1+1)~MA with \bils takes expected time $\Theta(n^2 + n^3/w^2)$ and the (1+1)~MA with \fils takes expected time $\Theta(n^3/w^2)$ to find the optimum. These times are polynomial for all choices of the hurdle width.

Note that the term $n^3/w^2$ decreases with the hurdle width, hence the surprising conclusion is that larger hurdle widths make the problem much harder for evolutionary algorithms, while making the problem easier for memetic algorithms.

The \hurdle problem, albeit having been defined for a very different purpose~\cite{PRUGELBENNETT2004135}, turns out to be an ideal example for showcasing the power of memetic algorithms and iterated local search.
This finding is particularly significant in the light of ``big valley'' structures, an important characteristic of many
hard problems from combinatorial optimisation~\cite{Ochoa2016,Reeves1999}, where ``many local optima
may exist, but they are easy to escape and the gradient,
when viewed at a coarse level, leads to the global optimum''~\cite{Hains2011}.
The \hurdle problem is a perfect and very illustrative example of a big valley landscape. By explaining how the (1+1)~MA easily solves the \hurdle problem class, we hope to gain insight into how memetic algorithms perform on big valley structures, which may help to explain why state-of-the-art memetic algorithms perform well on problems with big valley structures~\cite{Merz1998,Shi2017}.

%
%

\subsection{Related Work}

There are other examples of functions where memetic algorithms
were theoretically proven to perform well (see Sudholt~\cite{Sudholt2011c}
for a more extensive survey). In~\cite{Sudholt2006a} examples
of constructed functions were given where the (1+1)~EA, the (1+1)~MA,
and Randomised Local Search (\rls) can mutually outperform each other.
The paper~\cite{Sudholt2009} investigates the impact of the \emph{local
	search depth}, which is often used to limit the number of iterations
local search is run for. The author gives a class of example functions
where only specific choices for the local search depth are effective,
and other parameter settings, including plain evolutionary algorithms
without local search, fail badly. Similar results were obtained for the
choice of the \emph{local search frequency}, that is, how often local
search is run~\cite{Sudholt2006}.

Sudholt~\cite{Sudholt2010} showed for instances of classical
problems from combinatorial optimisation that memetic algorithms
with a different kind of local search, \emph{variable-depth local search},
can efficiently cross huge fitness valleys that are nearly impossible
to cross with evolutionary algorithms.
Witt~\cite{Witt2012a} further analysed the performance of a memetic
algorithm, iterated local search, for the \textsc{Vertex Cover} problem.
Sudholt and Zarges~\cite{Sudholt2010b} investigated the use of
memetic algorithms for the graph colouring problem.
Finally, Wei and Dineen analysed memetic algorithms for
solving the \textsc{Clique} problem, investigating the choice of
the fitness function~\cite{Wei2014a} as well as the choice of
the local search operator~\cite{Wei2014b}.

Gie{\ss}en~\cite{Giessen2013} presented another example function
class based on a discretised version of the well-known Rastrigin
function. He designed a memetic algorithm using a new local search
method called \emph{opportunistic local search}, where the search
direction switches between minimisation and maximisation
whenever a local optimum is reached. This function also resembles
a ``big valley'' structure in two dimensions as the bit string is mapped
onto a two-dimensional space.

Another line of research is work on \emph{hyperheuristics} that combine different operators. Alanazi and Lehre~\cite{Alanazi2014} demonstrated the usefulness of hyperheuristics for artificial functions, and Lissovoi, Oliveto, and Warwicker~\cite{Lissovoi2017a} presented novel, provably efficient hyperheuristic algorithms. The difference to memetic algorithms is that while hyperheuristics typically apply one operator, while learning which operator performs best, memetic algorithms apply different operators, variation and local search, in sequence. The interplay of variation and local search is a major challenge when analysing memetic algorithms.

\subsection{Outline}
The paper is structured as follows. Section~\ref{preliminaries} introduces the
(1+1)~EA, (1+1)~MA as well as the two local searches.
The class of \hurdle problems are then formally defined in
Section~\ref{hurdle-problems}, which also includes detailed description about their
properties. Section~\ref{hybridisation-needed} points out the inefficiency of
the (1+1)~EA and two local search algorithms  by investigating
their expected optimisation times on the \hurdle problems.
Next, tight bounds on the expected optimisation
time of the (1+1)~MA  are  derived in
Section~\ref{memetic-algorithm-efficient}, which
reveals the outperformance of the (1+1)~MA
to the alternative stochastic search methods.
Finally, concluding remarks are given in Section~\ref{conclusions}.

\section{Preliminaries}
\label{preliminaries}
\subsection{(1+1) Evolutionary Algorithm}

%

In order to focus on the main differences between evolutionary algorithms and memetic algorithms, and to facilitate a rigorous theoretical analysis, we consider simple bare-bones algorithms from these two paradigms.	

The (1+1)~EA is the simplest evolutionary algorithm, operating with a population of size one and using
only mutation. The mutation
operator flips each bit independently with mutation probability~$p_m$, with the default choice being~$p_m = 1/n$,
where $n$ is the length of the bitstring.
The fitness function is defined as $f \colon \mathcal{X} \rightarrow \mathbb{R}$,
where $\mathcal{X}=\{0,1\}^n$ is the binary search space, and has
to be maximised. Algorithm~\ref{11-ea} gives a full description of the
(1+1)~EA. Here $\mutate(x)$
returns a new bitstring resulting from
flipping bits in $x$ independently with probability
$p_m$.

\begin{algorithm}
	\DontPrintSemicolon
	\caption{(1+1)~EA}
	\label{11-ea}
	$x \sim \unif\{\mathcal{X}\}$\;
	\Repeat{some stopping condition is fulfilled.}{		
		$y \leftarrow \mutate(x,p_m)$\;
		\If {$f(y) \geq f(x)$}{
			$x \leftarrow y$	\;
		}	 			
	}
\end{algorithm}

Practical implementations of Evolutionary Algorithms in particular and other
search metaheuristics in general require to specify
some stopping condition. The simplest is to stop when a fixed
number of generations has been exceeded. The theoretical
results presented in this paper address the limiting case
when the algorithm runs until a global optimum is found.
In this case, we are interested in the
\textit{expected optimisation time} of the algorithm, defined
as the mean of the number of fitness (or function) evaluations
performed by the algorithm
until a global optimum is found.

\subsection{(1+1) Memetic Algorithm}

Algorithm~\ref{11-ma} \cite{Wei2014b} outlines the typical
procedure of the (1+1)~MA, the simplest memetic algorithm. The algorithm consists
of a population of one individual and produces an offspring in each
generation by independently flipping each bit in the current search point
with mutation probability $p_m=1/n$.
The newly generated offspring is then refined
further using a \ls.
Any local search algorithms can fit into the scenario.
Although the (1+1)~MA looks quite simple, it still captures the same
working principle as the general \mas. Analysing it can reveal
insights into how the general \mas operate and when they can
be employed to solve problems.

\begin{algorithm}
	\DontPrintSemicolon
	\caption{(1+1)~MA}
	\label{11-ma}				
	$x \sim \unif\{\mathcal{X}\}$\;	
	\Repeat{some stopping condition is fulfilled.}{			
		$y \leftarrow \mutate(x,1/n)$	\;
		$z \leftarrow \ls(y)$\;
		\If {$f(z) \geq f(x)$}{
			$x \leftarrow z $\;
		}
	}
\end{algorithm}	

\subsection{Local Searches}

We consider the following two local searches in the context of the (1+1)~MA. Both local searches are common practice and have also been analysed in~\cite{Dinneen2013}.

\begin{algorithm}
 	\DontPrintSemicolon
 \SetKwInOut{Input}{input}
 \SetKw{to}{to}
 \SetKw{iterations}{iterations}
 \caption{\fils}
 \label{first-ascent}
 \Input{a bitstring $x = (x_1,\ldots,x_n) \in \{0,1\}^n$}
		\For{$\delta$ \iterations}{
	create a random permutation \permute of set $\{1,2,\ldots,n\}$\;
	$\betterfound \leftarrow \textit{false}$\;
	\For{$i=1$ \to $n$}{
		$y \leftarrow \flip\left(x,\permute[i]\right)$\;	
		\If{$f(y) > f(x)$}{
			$x \leftarrow y$\;
			$\betterfound \leftarrow \textit{true}$\;
		}
	}
	\If{\betterfound = false}{
		return $x$\;	
	}
}
return $x$\;
\end{algorithm}

\subsubsection{First Improvement Local Search}(\fils), shown in Algorithm~\ref{first-ascent}, adapted from Wei and Dinneen~\cite{Dinneen2013}, takes advantage of the first improvement it finds while searching the neighbourhood.
The algorithm runs for $\delta$ iterations.
Bits are flipped according to a random permutation
$\permute$ of length $n$ (to avoid any search bias due to the choice of bit positions), and newly generated
individuals are then scored by the	fitness function.
Here $\flip(x,i)$ returns a new bitstring resulting from
flipping the $i$-th bit in  $x$.
The current search point is replaced by the
first neighbour found with a better fitness value.
The algorithm stops either after $\delta$
iterations of the outer for loop have been
performed or after visiting all $n$ neighbours of the current search point
without any improvement.

\subsubsection{Best Improvement Local Search}(\bils), shown in Algorithm~\ref{steepest-ascent}, adapted from Wei and Dinneen~\cite{Dinneen2013}, searches the whole neighbourhood and then picks a search point giving the best improvement.

The algorithm runs for $\delta$ iterations, and in each iteration
a neighbour with the largest improvement in the
fitness among all $n$ neighbours
of the current search point is picked to be
the next search point. In order to keep track of the progress so far, it
stores the best neighbour(s) and best fitness into \curbestind and
\curbestfit, respectively. This means that
whenever a neighbour with better fitness compared to
\curbestfit has been found, the algorithm performs
update on the two variables. At the end of an iteration, if there are
more than one neighbours with the same fitness value that is
better than $f(x)$, then the next search point is chosen uniformly at random
from the set \curbestind.

\begin{algorithm}
	\DontPrintSemicolon
	\SetKwInOut{Input}{input}
	\SetKw{to}{to}
	\SetKw{iterations}{iterations}
	\caption{\bils}
	\label{steepest-ascent}
	\Input{a bitstring $x = (x_1,\ldots,x_n) \in \{0,1\}^n$}
	\For{$\delta$ \iterations}{
		$\curbestind \leftarrow \emptyset$\;
		$\curbestfit \leftarrow f(x)$\;
		\For{$i=1$ \to $n$}{
			$y \leftarrow \flip(x,i)$\;	
			\uIf{$f(y) >\curbestfit$}{
				$\curbestind \leftarrow \{y\}$\;
				$\curbestfit \leftarrow f(y)$\;
			}
			\ElseIf{$f(y) =\curbestfit$}{
				$\curbestind \leftarrow \curbestind \cup \{y\}$\;
			}		
		}
		\uIf{$\curbestind = \emptyset$}{
			return $x$\;
		}
		\Else{
			$x \sim \unif\{\curbestind\}$\;	
		}
	}
	return $x$\;
\end{algorithm}

%
\section{Class of \hurdle Problems}
\label{hurdle-problems}%

The \hurdle function class was introduced back in 2004 by Pr{\"u}gel-Bennett~\cite{PRUGELBENNETT2004135} as an example class where genetic algorithms with crossover outperform hill climbers. Here we give a formal definition and discuss basic properties of the function that will be used in the subsequent analyses.

The objective is to find a bitstring that maximises the fitness
function ${f \colon \mathcal{X} \to \mathbb{R}}$.
The value of the fitness function at a given bitstring
$x \in \mathcal{X}$ is \cite{PRUGELBENNETT2004135}
$$
f(x) = - \ceil*{\frac{z(x)}{w}} - \frac{\rem(z(x),w)}{w}.
$$
In this function, $z(x)$ is the number of zeros in the
bitstring $x$.
$w \in \{2, 3, \dots, n\}$ is called the hurdle width and is the only parameter of
the \hurdle problems. Note that $w = w(n)$ may be a function of~$n$. Finally, $\rem(z(x),w)$ is the remainder of $z(x)$ divided by $w$, while
$\ceil{\cdot}$ is the ceiling function.

\begin{lemma}
	The global optimum for the \hurdle problem is $1^n$.
\end{lemma}
\begin{proof}
	For every $x \in \{0, 1\}^n$, $f(x)\le 0$ since
    both $z(x)$ and $\rem(z(x),w)$ cannot be
    negative. The equality happens if and only if both
    $z(x)$ and $\rem(z(x),w)$ equal zero, or equivalently
    $x=1^n$.
\end{proof}
\begin{figure}
		\centering
		\includegraphics[width=.58\linewidth]{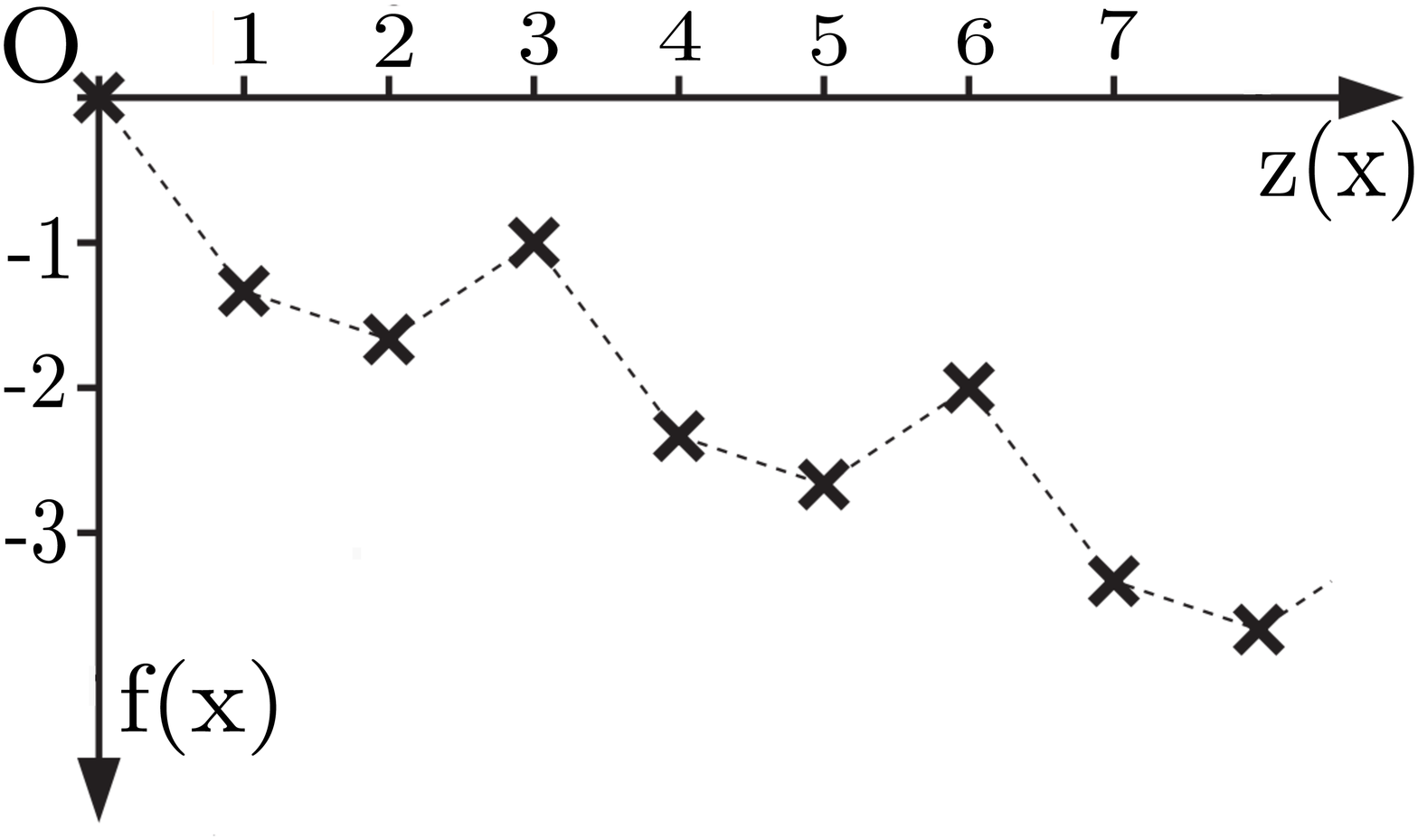}
	\caption{Fitness landscapes of \hurdle with $w=3$.}
	\label{fig:hurdle-landscape}
\end{figure}
Note in particular that $z(x)$ can also
be viewed as the Hamming distance $H(x,1^n)$ between
the current solution~$x$ and the global optimum $1^n$.
The fitness landscapes close to the global optimum are
shown in Fig.~\ref{fig:hurdle-landscape} \cite{PRUGELBENNETT2004135}.
It can be clearly seen that the global optimum $1^n$ coincides with
the origin where both $z(1^n)=0$ and $f(1^n)=0$.
In the following lemma, the term \emph{nearest} refers to the scale of $z(x)$, i.\,e.\ the most similar number of zeros. Note that this relates to Hamming distances as follows: any search point with $z(x)$ zeros has Hamming distance at least $|i - z(x)|$ to any search point with $i$ zeros. A sufficient condition for a mutation of $x$ having $i$ zeros is flipping $i - z(x)$ zeros and no other bits if $i \ge z(x)$ or flipping $z(x) - i$ ones if $i \le z(x)$.

\begin{lemma}
	\label{hurdle-property}
	Given a \hurdle problem with hurdle width $w$
	and a local optimum $x \neq 1^n$ as the current search point,
    the nearest search points with fitness larger than $f(x)$ are all search points with $z(x) - w$ zeros: $\{x' \mid z(x') = z(x) - w\}$.
\end{lemma}
\begin{proof}
	
	\begin{figure}[!ht]
		\centering
		\includegraphics[scale=.21]{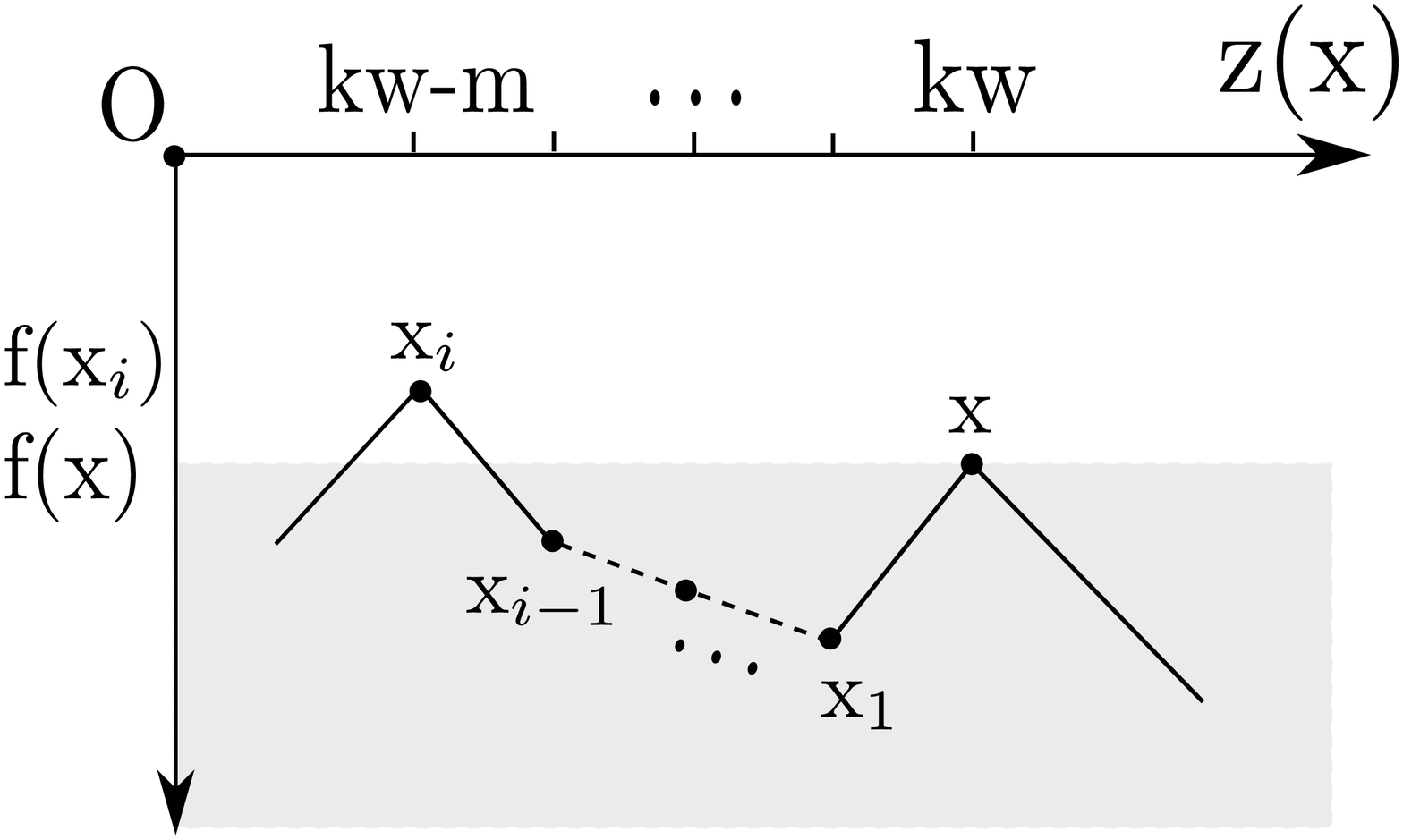}	
		\caption{Fitness landscape of \hurdle problem with arbitrary $w$.}
		\label{fig:best-on-hurdle}				
	\end{figure}
	The current local optimum $x\neq 1^n$contains
	$z(x) = kw$ zeros where $k \in \mathbb{N}\setminus\{0\}$ (see Fig.~\ref{fig:best-on-hurdle}).
	Let us consider a search point $x_i$ which is the nearest search point
	with 	$f(x_i) \geq f(x)$ and $z(x_i)=kw-m < z(x)$
	where $0< m \leq w$. Here, we exclude the case $m>w$ as
	the next local optimum corresponds to $m=w$, and its fitness value
	is already known to be better than $f(x)$.
	
	Now we  need to calculate the fitness values for two search points, $x$ and $x_i$.
	Note that $z(x)/w = kw/w = k$, and $\rem(z(x),w) = \rem(kw,w) = 0$, then
	$$
	f(x) = - \ceil*{\frac{z(x)}{w}} - \frac{\rem(z(x),w)}{w}=-\ceil*{k}-\frac{0}{w} =  -k.
	$$
	
	On the other hand, $z(x_i)/w = (kw-m)/w = k-m/w$, and
	$\rem(z(x_i),w) =\rem(kw-m,w)=w-m$ as we can rewrite
	$z(x_i) = kw-m = (k-1)w+(w-m)$, then
	$$
	f(x_i) = - \ceil*{k-\frac{m}{w}} - \frac{w-m}{w} =  - \ceil*{k-\frac{m}{w}} - 1+\frac{m}{w}.
	$$
	
	Now we consider two different cases as follows.
	If $m=w$, then $m/w=1$ and $f(x_i) = -(k-1)-1+1 = -k+1 >  f(x)$;
	otherwise, $m/w <1$, and $1-\frac{m}{w}>0$, then
	$f(x_i) = -k-\left(1-\frac{m}{w}\right) < f(x)$.
	
	For all $0<m \leq w$, we only have $f(x_i) >f(x)$ if and
	only if $m=w$, and then $z(x_i)=(k-1)w$
	where $k \in \mathbb{N}\setminus\{0\}$.
	This result implies that $x_i$ must be a local optimum.
	This proof also shows that the difference
	in the fitness values of two consecutive local optima
	is exactly one.
\end{proof}

\section{Why Is Hybridisation Necessary?}
\label{hybridisation-needed}
\subsection{Local Searches}
In this section, we show that local search algorithms in general are unable to
optimise  the \hurdle problems, unless the initial search point is chosen from a
specific regions in the fitness landscape.

Let $z$ denote the number of zeros in the initial search point.
It is obvious that if $z\geq w$, then the local search algorithm cannot locate the
global optimum as it gets stuck at a local one forever. Otherwise, the global optimum
can be found with some probability. However, if the
local search is allowed to run only once, then
for $w \ll n/2$ the chance that it can optimise the \hurdle problems is close to zero since the
number of search points with at most $w-1$ zeros is
significantly smaller compared to the size of the search space, i.e. $2^n$.	
One way to overcome this problem is to employ a \textit{restart} mechanism, which restarts the local search algorithm once a local optimum has been found.

\begin{theorem}\label{LS-ineficient}
	The expected optimisation time of local search
	algorithms \bils and \fils with $\delta \ge w$,
    restarting after $\delta$ iterations of the local search,
	on \hurdle problems with hurdle width $w\leq cn$ for some constant $c <1/2$
	is $2^{\Omega(n)}$.
\end{theorem}

	We focus on $w \leq cn$ for some constant $c<1/2$ as, otherwise,
	the majority of search points would lie
	in the basin of attraction of the global optimum,
	resembling the function \om\footnote{The ones-counting
		problem, i.e. $\om(x):= \sum_i x_i$.}.
	
\begin{proof}[Proof of Theorem~\ref{LS-ineficient}]
	The local search algorithm flips one bit and only accepts new
	search points with strictly better fitness value compared to the current one
	in each iteration; therefore, the initial search point
	\textit{decides} whether the global optimum can be
	reached. It is clear that this search point needs to
	have at most $cn-1$ zeros in order for the algorithm
	to be able to optimise the problem (see Fig.~\ref{fig:hurdle-landscape}).
	By Chernoff bounds \cite{Motwani1995}, this
	event happens with probability at most
	$2^{-\Omega(n)}$.
	The expected number of restarts until this event happens
	is at least $2^{\Omega(n)}$.
Since every restart clearly leads to at least one function evaluation,
the expected optimisation time of local search
	algorithms on \hurdle problems is $2^{\Omega(n)}$.
\end{proof}

\subsection{(1+1) Evolutionary Algorithm}
In this section, we prove a tight bound $\Theta(n^{w})$
on the expected optimisation time of the (1+1)~EA on the \hurdle
problems with an arbitrary hurdle width $2 \le w \le n/2$.
This result implies that the (1+1)~\ea is not efficient on \hurdle. Our rigorous analysis complements the non-rigorous arguments given in~\cite{PRUGELBENNETT2004135}.
Note in particular that starting
from an initial search point with at least $w$ zeros,
    the first generation will end in a local optimum
    after at most $w$ iterations of the local search with $\delta\ge w$.
    Since this has a small
    \textit{additive} contribution to the overall runtime, we can assume that
    a local optimum is the current search point.

\begin{theorem}
	The expected optimisation time of the (1+1)~EA on the
	\hurdle problem with hurdle width $w$ is $\mathcal{O}(n^w)$.
\end{theorem}
\begin{proof}
	Assume that a local optimum $x$ be the current search point with $z$ zeros.
Lemma~\ref{hurdle-property} yields that the nearest
	search points with a better fitness  than $f(x)$ are all local
	optima with $z-w$ zeros. For such a mutation we just need to flip
	at least $w$ zeros simultaneously, and keep $n-w$ remaining bits unchanged.
	The probability of flipping $w$ bits is given by $\left(1/n\right)^w$, and
	keeping $n-w$ bits unchanged is $\left(1-1/n\right)^{n-w}$.
	Therefore, the probability of obtaining a better solution is bounded from below by
	$p_z \geq \binom{z}{w} \left(\frac{1}{n}\right)^w \left(1-\frac{1}{n}\right)^{n-w}
	\geq \frac{1}{en^w}\left(\frac{z}{w}\right)^w$
    since $\left(1-\frac{1}{n}\right)^{n-1}\ge 1/e$ and
    $\binom{z}{w}\ge (z/w)^w$ \cite[Appendix C]{Cormen2009}.
	The expected number of generations until the global optimum found
	is now bounded from above by
	\begin{equation}
	\label{eq:fitness-level-bound}
	\mathbb{E}\left[T\right]
	\leq \sum_{z=w}^{n}\frac{1}{p_z}
	\leq en^w w^w\sum_{z=w}^{n}\frac{1}{z^w}.
	\end{equation}
	Now we need to calculate the sum $\sum_{z=w}^{n}\frac{1}{z^w}$.
	Let us consider the function
	$
	g(z) = z^{-w},~ \forall z \in [w,n].
	$
	Since
	$
	\nabla_z g=-w/z^{w+1} < 0$,
	$g$~is a monotonically decreasing function in $[w,n]$. So
	we can approximate the upper bound of this sum using a
	method called \textit{approximation by integrals}, i.e.
	$
	\sum_{x=a}^{b}g(x) \leq \int_{a-1}^{b}g(x)dx
	$
     \cite[Appendix A]{Cormen2009}.
	Applying this method yields
	\begin{align*}
		\sum_{z=w}^{n}\frac{1}{z^w} =\;& w^{-w} + \sum_{z=w+1}^{n}\frac{1}{z^w} \\
		\leq\;&  w^{-w} + \int_{w}^{n}\frac{1}{z^w}dz\\
		\le\;& w^{-w} + \int_{w}^{\infty}\frac{1}{z^w}dz = w^{-w} + w^{-w} \cdot \frac{w}{w-1} \le 3w^{-w}.
	\end{align*}
	Substituting this into~\eqref{eq:fitness-level-bound},
	we get $\mathbb{E}\left[T\right] \leq 3en^w$.
	%
\end{proof}

To derive the lower bound, we again focus on $w \le n/2$ for the
same line of arguments as in Theorem~\ref{LS-ineficient}.

\begin{theorem}
	\label{the:lower-bound-EA}
	The expected optimisation time of the (1+1)\;\ea on \hurdle problems
	with hurdle width $2 \le w \le n/2$ is $\Omega(n^w)$.
\end{theorem}
\begin{proof}
	We first show that a search point with $w$ zeros is reached
	with probability~$\Omega(1)$. The initial number of zeros
    follows a binomial 	distribution with parameters
    $n$ and $1/2$. As $w \le n/2$, the probability
	that the initial search point will have at least $w$ zeros is at
	least $1/2$ (by symmetry of the binomial distribution). It
	may be possible to ``jump over'' search points with
	$w$ zeros, \ie to make a transition from $i > w$ zeros
	to $j < w$ zeros, so long as the \hurdle value with $i$ zeros
	is less or equal to that of $j$ zeros. However, Lemma~9
	in~\cite{Paixao2016} states that the conditional probability
	of standard bit mutation reaching a search point with $w$ zeros, given that a
	search point with \emph{at most} $w$ zeros is reached,
	is at least~$1/2$.
	
	Once the (1+1)\;\ea has reached such a search point with $w$ zeros,
	the expected remaining optimisation time is
	${n^w (1-1/n)^{-n+w}} = \Omega(n^w)$ as the optimum is
	the only search point with a higher fitness
	(see Lemma~\ref{hurdle-property}) and the probability
	of jumping to the optimum from any such search point is
    $n^{-w} (1-1/n)^{n-w}$. By the law of total expectation, the expected optimisation time
    is at least $\Omega(1)\cdot \Omega(n^w) = \Omega(n^w)$.
\end{proof}

We remark that the (1+1)~\ea can be slightly sped up
by increasing the mutation rate. As the above analysis has shown, the expected optimisation
time is dominated by the time to locate the global optimum
from a search point with Hamming distance~$w$ to it.
From this starting point, a mutation rate of $w/n$ maximises
the probability of flipping exactly $w$ bits. In fact, this choice
was already used in~\cite{PRUGELBENNETT2004135}
for the (1+1)~EA. However, even choosing an optimal
mutation rate does not help much as, even if a mutation
does flip exactly $w$ bits, the mutation needs to select
the right bits to flip, and the chance of choosing exactly
the $w$ bits that differ from the optimum is $1/\binom{n}{w}$,
leading to an expected time of at least
$\binom{n}{w} \ge n^w/w^w$ \cite{Cormen2009} from such a local optimum
(and $\Omega(\binom{n}{w}) \ge \Omega(n^w/w^w)$ from
random initialisation). This still results in a superpolynomial
expected time if $w = \omega(1)$, that is, $w$ grows with~$n$.

\section{Memetic Algorithms Are Efficient}
\label{memetic-algorithm-efficient}


We now show that, in contrast to local search on its own, and evolutionary algorithms, the (1+1)~MA can find the global optimum efficiently, for both \bils and \fils. The main result of this section is as follows.
Note in particular we consider \bils
and \fils with local search depth $\delta \ge w$ as this is sufficient to run into local optima from anywhere in the search space.

\begin{theorem}
	\label{the:number-of-evaluations-MA-general-w}
	The expected number of function evaluations of
	the (1+1)~MA on \hurdle with any hurdle width $2 \le w \le n/2$ is
	$\Theta(n^2 + n^3/w^2)$ for \bils and $\Theta(n^3/w^2)$ for \fils, both using $\delta \ge w$.	
\end{theorem}

In order to prove Theorem~\ref{the:number-of-evaluations-MA-general-w}, we first prove upper bounds of $\mathcal{O}(n^2 + n^3/w^2)$ and $\mathcal{O}(n^3/w^2)$, respectively, and then we prove lower bounds of $\Omega(n^2 + n^3/w^2)$ and $\Omega(n^3/w^2)$, respectively, showing that the upper bounds are asymptotically tight.

To prove the upper bounds, we first provide an upper bound on the expected number of generations needed. This does not include the function evaluations made during local search, which will be bounded separately.
\begin{theorem}
	\label{the:number-of-generations-MA-general-w}
	The expected number of generations of
	the (1+1)~MA using \bils or \fils with $\delta \ge w$ on \hurdle with any hurdle width $2 \le w \le n$ is
	$\mathcal{O}(n^2/w^2)$.	
\end{theorem}
\begin{proof}
	Assume a local optimum
    $x$ is the current search point with $z$
	zeros in the bitstring. The fitness landscape around $x$ is illustrated in
	Fig.~\ref{fig:hurdle-problem-ma-w-1}.	
	Lemma~\ref{hurdle-property} yields that
    the nearest search points with a better fitness than $f(x)$ are the local
	optima with exactly $z-w$ zeros.
	Let us consider the situation when the (1+1)~MA flips at least two zeros
	to move from $x$ to $x_2$ (say event A), and then any
	local search will locate the next local optimum $x_w$ by performing
	a sequence of one-step jumping:
	$x_2$ to $x_3$,\ldots, $x_{w-1}$ to $x_w$  (say event B).
		\begin{figure}
		\centering
		\includegraphics[scale=.060]{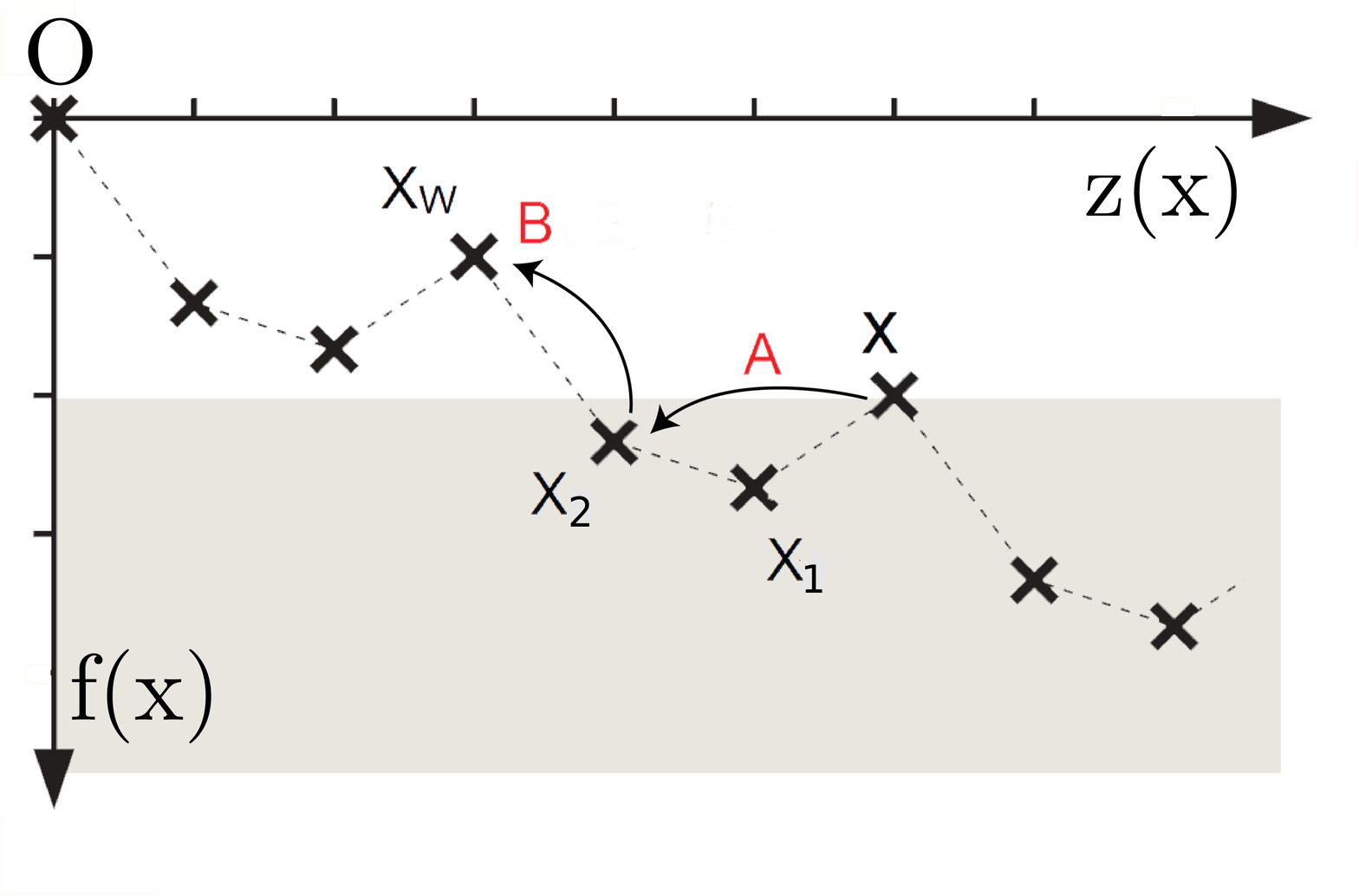}
		\caption{(1+1)~MA on \hurdle problems}
		\label{fig:hurdle-problem-ma-w-1}	
	\end{figure}
	Clearly, event A happens with probability
	$$ p(A)=\binom{z}{2}\left(\frac{1}{n}\right)^2\left(1-\frac{1}{n}\right)^{n-2}
	\geq \frac{z(z-1)}{2en(n-1)}.$$	
	Given event A, event B occurs with unity probability, i.\,e.\ ${p(B\mid A)=1}$,
	as the local search always locates the next local optimum.
	Hence, the probability of reaching $x_w$ from $x$
	is bounded from below by $p(A)$, and		
	the expected number of generations is at most $1/p(A)$.
	The expected number of generations until the global optimum is found is
	\begin{align}
		\label{ma-expected-generations}
		\mathbb{E}\left[T\right]
		&\leq \sum_{z\in \{w,2w,\ldots\}}\frac{2en(n-1)}{z(z-1)}.
	\end{align}
	We have, using $\sum_{i=1}^\infty \frac{1}{i^2} = \pi^2/6$,
	\begin{align*}
		\sum_{z\in \{w,2w,\ldots\}}\frac{1}{z(z-1)} &\leq \sum_{i=1}^{\infty} \frac{1}{iw(iw-1)}\\
		&\le \sum_{i=1}^{\infty} \frac{1}{iw(iw-iw/2)}
		= \frac{2}{w^2} \sum_{i=1}^{\infty} \frac{1}{i^2}
		= \frac{\pi^2}{3w^2}.
	\end{align*}	
	Substituting into (\ref{ma-expected-generations})
	yields $\mathbb{E}\left[T\right] = \mathcal{O}(n^2/w^2)$.
\end{proof}	 	


In order to bound the number of function evaluations made during local search, we distinguish between local searches that result in a strict improvement over the previous current search point, and those that don't. This is an example of the \emph{accounting method}~\cite[Chapter 17.2]{Cormen2009}, where function evaluations are charged to one of two accounts and the total costs are bounded separately for each account.
Adding the two bounds will yield an upper bound on the total number of function evaluations made during local search.

We first bound the number of function evaluations spent in any improving local search.
\begin{lemma}
	\label{lem:improving-local-searches}
	Call a local search \emph{improving} if it terminates with a search point that has a strictly better fitness than the current search point of the (1+1)~MA; otherwise, the local search is called \emph{non-improving}.
	The number of function evaluations spent in any improving local search call on \hurdle with hurdle width~$w$ and $\delta \ge w$ is at most~$wn$ for \bils and at most $2n$ for \fils.
\end{lemma}
\begin{proof}
	As \hurdle is a function of unitation, while no local optimum is found, local search will either decrease the number of ones in each iteration, or increase the number of ones in every iteration. In both cases a local optimum will be found after at most~$w-1$ iterations. Once a local optimum has been
    reached, at most $n$ further evaluations are needed before local search stops.
	
	As \bils makes $n$ function evaluations in every iteration, it makes at most $n(w-1)+n=wn$ function evaluations in total.
	
	For \fils, after one iteration of the outer for loop (see Algorithm~\ref{first-ascent})
	a local optimum will be found, and then $n$ further evaluations are needed before it stops.
%
\end{proof}

The expected number of function evaluations for non-improving local searches can be bounded as follows.
\begin{lemma}
	\label{lem:non-improving-local-searches}
	In the setting of Theorem~\ref{the:number-of-evaluations-MA-general-w}, the expected number of function evaluations spent by \bils and \fils during any non-improving local search is $\Theta(n)$.
\end{lemma}
\begin{proof}
	The lower bound $\Omega(n)$ is trivial as both local searches make at least $n$ function evaluations before stopping.
	
	Let $i$ denote the number of zeros in the current search point of the (1+1)~MA and $j$ denote the number of zeros in the search point after mutation, from which local search is called.
	
	If $\rem(i, w) = 0$, that is, $i$ is a local optimum, $j < i-1$ will lead to an improving local search
	(see Fig.~\ref{fig:hurdle-landscape}), and $j=i-1$ may either be improving or go back to a search point with $i$ ones in one iteration. If $j \ge i$ then local search will make at most $j-i$ iterations.
	
	If $\rem(i, w) > 0$, that is, $i$ is not a local optimum, $j < w \cdot \lceil i/w \rceil$ leads to an improving local search, whereas $j \ge w \cdot \lceil i/w \rceil$ will stop after at most $j-i$ iterations.
	
	In all these cases, local search is either improving, or it makes at most $|j-i|$ iterations. Note that a necessary condition of mutating a search point with $i$ ones into one with $j$ ones is that at least $|j-i|$ bits flip. The probability for this event is at most $\binom{n}{|j-i|} n^{-|j-i|} \le 1/(|j-i|!)$. The expected number of iterations in a non-improving local search is thus at most
	\[
	\sum_{j=0}^n |j-i| \cdot \frac{1}{|j-i|!}
	\le 2\sum_{d=1}^\infty d \cdot \frac{1}{d!}
	= 2\sum_{d=1}^\infty \frac{1}{(d-1)!}
	= 2\sum_{d=0}^\infty \frac{1}{d!} = 2e.
	\]
	The number of function evaluations made during a
    local search that stops after $s$ iterations
    is at most $(s+1)n$. Hence the expected number
    of function evaluations is at most $(2e+1)n$.
\end{proof}

Putting the previous results together, we are now prepared to prove the upper bounds claimed in Theorem~\ref{the:number-of-evaluations-MA-general-w}.
\begin{proof}[Proof of the upper bounds from Theorem~\ref{the:number-of-evaluations-MA-general-w}]
	By Theorem \ref{the:number-of-generations-MA-general-w} the expected number of generations is bounded by $\mathcal{O}(n^2/w^2)$. The expected number of function evaluations spent in any non-improving local search is $\mathcal{O}(n)$ according to Lemma~\ref{lem:non-improving-local-searches}. Together, the number of function evaluations in all non-improving local searches is at most $\mathcal{O}(n^3/w^2)$.
	
	By Lemma~\ref{lem:improving-local-searches}, the number of function evaluations in any improving local search is at most~$wn$ for BILS and at most $2n \le wn$ for FILS. Since every improving local search ends in a local optimum with a better fitness than the current search point of the (1+1)~MA, there can only be $\mathcal{O}(n/w)$ improving local searches as this is a bound on the number of fitness levels containing local optima. Hence the effort in all improving local searches is bounded by~$\mathcal{O}(n^2)$, and the overall number of function evaluations is bounded by $\mathcal{O}(n^2 + n^3/w^2)$.
\end{proof}

To prove the lower bounds from Theorem~\ref{the:number-of-evaluations-MA-general-w}, we first show a very general lower bound of $\Omega(n^2)$ for the (1+1)~MA with \bils. It holds for all functions with a unique global optimum and may be of independent interest.
\begin{theorem}
	\label{the:lower-bound-for-bils}
	The (1+1)~MA using \bils makes at least $\Omega(n^2)$ function evaluations, with probability $1-2^{-\Omega(n)}$ and in expectation, on any function with a unique global optimum.
\end{theorem}
\begin{proof}
	It suffices to show the high-probability statement as the expectation is at least $(1-2^{-\Omega(n)}) \cdot \Omega(n^2) = \Omega(n^2)$.
	
	We show that with probability $1-2^{-\Omega(n)}$ one of the following events occurs.
	\begin{enumerate}
		\item[$A$:] The (1+1)~MA spends at least $n/12$ generations before finding the optimum.
		\item[$B$:] \bils makes a total of at least $n/6$ iterations before finding the optimum.
	\end{enumerate}
	Each event implies a lower bound of $\Omega(n^2)$ as each iteration of \bils makes $n$ function evaluations, and each generation leads to at least one iteration of \bils.
	
	In order for none of these events to occur, the (1+1)~MA must find the optimum within $n/12$ generations, using fewer than $n/6$ iterations of \bils in total. For this to happen, one of the following rare events must occur:
	\begin{enumerate}
		\item[$E_1$:] the (1+1)~MA is initialised with a search point that has a Hamming distance less than $n/3$ to the unique optimum or
		\item[$E_2$:] the initial search point has a Hamming distance of at least $n/3$ to the optimum, and the algorithm decreases this distance to~0 during the first $n/12$ generations, using fewer than $n/6$ iterations of \bils.
	\end{enumerate}
The reason is that, if none of the events $E_1$ and $E_2$ occur, then this implies $A \cup B$.
By contraposition, $\overline{A \cup B} \Rightarrow E_1 \cup E_2$ and $\Prob{\overline{A \cup B}} \le \Prob{E_1 \cup E_2} \le \Prob{E_1} + \Prob{E_2}$ by the union bound.

	Event $E_1$ has probability $\Prob{E_1} \le 2^{-\Omega(n)}$ by Chernoff bounds.
	
	For $E_2$, note that each iteration of local search can decrease the Hamming distance to the optimum by at most~1. Hence all iterations of \bils can only decrease the Hamming distance to the optimum by $n/6$ in total, and so the remaining distance of $n/3 - n/6 = n/6$ needs to be covered by mutations. Each flipping bit can decrease the distance to the optimum by at most~1. We have at most $n/12$ mutations, hence the expected number of flipping bits is at most $n/12$. The probability that at least $n/6$ bits flip during at most $n/12$ mutations is $2^{-\Omega(n)}$, which follows from applying Chernoff bounds to iid indicator variables $X_{i, t} \in \{0, 1\}$ that describe whether the $i$-th bit is flipped during generation~$t$ or not. Hence $\Prob{E_2} \le 2^{-\Omega(n)}$.
	
	Together, we have by the union bound,
	\[
	\Prob{\overline{A \cup B}} \le \Prob{E_1} + \Prob{E_2} \le 2^{-\Omega(n)} + 2^{-\Omega(n)} \le 2^{-\Omega(n)}.
	\]
	This completes the proof.
\end{proof}

\begin{proof}[Proof of the lower bounds from Theorem~\ref{the:number-of-evaluations-MA-general-w}]
	A bound of $\Omega(n^2)$ for the (1+1)~MA with \bils follows from Theorem~\ref{the:lower-bound-for-bils}.
	
	We prove lower bounds $\Omega(n^3/w^2)$ for both local searches by considering the remaining time when the (1+1)~MA has reached a local optimum with $w$ zeros. Theorem~\ref{the:lower-bound-EA} reveals that the (1+1)~EA reaches such a local optimum with probability~$\Omega(1)$, and it is obvious that the same statement also holds for the (1+1)~MA. Then a lower bound of $\Omega(n^3/w^2)$ follows from showing that the expected number of function evaluations starting with a local optimum having $w$ zeros is $\Omega(n^3/w^2)$.
	
	From such a local optimum, the (1+1)~MA with \bils has to flip at least two zeros in one mutation. Otherwise, the offspring will have at least $w-1$ zeros, and \bils will run back into a local optimum with $w$ zeros (or a worse local optimum). The probability for such a mutation is at most $\binom{w}{2} \cdot 1/n^2 = \mathcal{O}(w^2/n^2)$, and the expected number of generations until such a mutation happens is at least $\Omega(n^2/w^2)$.
	
	The same statement holds for the (1+1)~MA with \fils: here it is necessary to either flip at least two zeros as above, or to create a search point with $w-1$ zeros and to have \fils find a search point with $w-2$ zeros as the first improvement. In the latter case, \fils will find the global optimum. The probability of creating a search point with $w-1$ zeros is at most $w/n$ as it is necessary to flip one of $w$ zeros. In this case \fils creates a search point with $w-2$ ones as first improvement if and only if the first bit to be flipped is a zero. Since there are $w-1$ zeros, and each bit has the same probability of $1/n$ of being the first bit flipped, the probability of the first improvement decreasing the number of zeros is $(w-1)/n$. Together, the probability of a generation creating the global optimum is still $\mathcal{O}(w^2/n^2)$, and the expected number of generations is still at least $\Omega(n^2/w^2)$.
	
	In every generation, both \bils and \fils make at least $n$ evaluations. Hence we obtain $\Omega(n^3/w^2)$ as a lower bound on the number of function evaluations.
\end{proof}

\section{Conclusions and Future Work}
\label{conclusions}

We have provided a rigorous runtime analysis, comparing the
simple (1+1)~EA with the (1+1)~MA using two local
search algorithms, \fils and \bils, on the class of \hurdle problems.
Our main results are tight bounds of $\Theta(n^2+ n^3/w^2)$ on the expected number of function evaluations of the (1+1)~MA using \bils and
$\Theta(n^3/w^2)$ for the (1+1)~MA using \fils. On the other hand,
the (1+1)~EA and local search algorithms on their own take time $\Theta(n^w)$ and
$2^{\Omega(n)}$, respectively. For $w = \omega(1)$ the latter times are superpolynomial, whereas the expected number of function evaluations for the (1+1)~MA is always polynomial, regardless of the hurdle width~$w$.

The \hurdle problem hence represents an illustrative problem where a hybrid algorithm drastically outperforms both of the individual search algorithms it contains, when these are run on their own.

A surprising conclusion is also that the \hurdle problem class becomes easier for the (1+1)~MA as the hurdle width~$w$ grows. The reason is that while the (1+1)~EA has to jump to the global optimum by mutation, for the (1+1)~MA it suffices to jump into the basin of attraction of the global optimum. Increasing the hurdle width~$w$ makes it harder for the (1+1)~EA to make this jump, but it also increases the size of the basin of attraction of the global optimum, effectively giving the (1+1)~MA a bigger target to jump to.

More specifically, our analysis has shown that the (1+1)~MA can efficiently reach a better local optimum by flipping two 0-bits during mutation, as then the resulting mutant is located in the basin of attraction of a better local optimum. The expected optimisation time is dominated by the time spent in the last local optimum, that is, when the current search point contains $w$ zeros. From here, a mutation flipping two 0-bits has probability $\Theta(w^2/n^2)$, where the factor of $w^2$ results from $\binom{w}{2}$ choices for the two flipping 0-bits. The larger the hurdle width, the larger the probability of making such a mutation, and the lower the term of order $n^3/w^2$ becomes. Note that the (1+1)~MA is otherwise agnostic to the width of the fitness valley, as local search will efficiently locate a better local optimum, regardless of the distance between local optima. This is in sharp contrast to the (1+1)~EA, which has to flip exactly $w$ bits to jump to the optimum, leading to an expected time of $\Theta(n^w)$.

Amongst problems with a ``big valley'' structure, \hurdle has a  favourable landscape for the (1+1)~MA as local optima have a very small Hamming distance to search points in the basin of attraction of better local optima. This makes it easy to transition from one local optimum to another by mutation and local search. A promising avenue for future work would be to analyse the performance of the (1+1)~MA for other classes of problems with big valley structures where larger jumps need to be made to transition to better local optima. This may require increasing the mutation rate, as commonly done in iterated local search algorithms~\cite{Lourencco2002}.

Another avenue for future work could be to rigorously analyse the expected running time of genetic algorithms with crossover on the \hurdle problem class, to investigate how their performance compares against that of the (1+1)~MA. Experimental results in~\cite{PRUGELBENNETT2004135} suggest that crossover provides a substantial advantage over the (1+1)~EA, however no rigorous analysis has been done.

\end{document}